\documentclass[letterpaper, 10 pt, conference]{format/ieeeconf}

\usepackage{amsmath}
\usepackage{amssymb}
\usepackage{amsfonts}

\usepackage{amsthm}
\usepackage[ruled,vlined]{algorithm2e}
\usepackage{mathtools}
\usepackage{tabularx}
\usepackage{graphicx}
\usepackage{subcaption} 
\usepackage{enumerate}
\usepackage{booktabs}
\usepackage{float}
\usepackage{url}
\usepackage{verbatim}
\usepackage{booktabs} 
\usepackage{multirow}
\usepackage{color}
\usepackage{pifont}
\usepackage[linkcolor=black,citecolor=black,urlcolor=black,colorlinks=true]{hyperref}
\usepackage{cite}
\usepackage{mathrsfs}
\usepackage{xcolor}

\graphicspath{{figures/}}
\IEEEoverridecommandlockouts
\newtheorem{lemma}{Lemma}

\title{LF-3PM: a LiDAR-based Framework\\for Perception-aware Planning with Perturbation-induced Metric}

\author{{Kaixin Chai$^{{\dag}2,3}$, Long Xu$^{{\dag}1,3}$, Qianhao Wang$^{1,3}$, Chao Xu$^{1,3}$, Peng Yin$^{2}$, and Fei Gao$^{1,3}$}
 \thanks{This work was supported by the ”Pioneer” and ”Leading Goose” R\&D Program of Zhejiang under Grant 2024C01170 and the National Natural Science Foundation of China under grant no. 62322314.}
 \thanks{Corresponding author: Fei Gao.}
 \thanks{$^{\dag}$Indicates equal contribution.}
 \thanks{$^{1}$State Key Laboratory of Industrial Control Technology, Zhejiang University, Hangzhou 310027, China.}
 \thanks{$^{2}$City University of Hong Kong, Hong Kong, China.}
 \thanks{$^{3}$Huzhou Institute of Zhejiang University, Huzhou 313000, China.}
 \thanks{E-mail: {\tt\small kaixinchai@outlook.com, \{gaolon,fgaoaa\} @zju.edu.cn}}
}

\begin{document}
    \maketitle
    \thispagestyle{empty}
    \pagestyle{empty}
\begin{abstract}
Just as humans can become disoriented in featureless deserts or thick fogs, not all environments are conducive to the Localization Accuracy and Stability (LAS) of autonomous robots. This paper introduces an efficient framework designed to enhance LiDAR-based LAS through strategic trajectory generation, known as Perception-aware Planning. Unlike vision-based frameworks, the LiDAR-based requires different considerations due to unique sensor attributes. Our approach focuses on two main aspects: firstly, assessing the impact of LiDAR observations on LAS. We introduce a perturbation-induced metric to provide a comprehensive and reliable evaluation of LiDAR observations. Secondly, we aim to improve motion planning efficiency. By creating a Static Observation Loss Map (SOLM) as an intermediary, we logically separate the time-intensive evaluation and motion planning phases, significantly boosting the planning process. In the experimental section, we demonstrate the effectiveness of the proposed metrics across various scenes and the feature of trajectories guided by different metrics. Ultimately, our framework is tested in a real-world scenario, enabling the robot to actively choose topologies and orientations preferable for localization. The source code is accessible at https://github.com/ZJU-FAST-Lab/LF-3PM.

\end{abstract}

\section{Introduction}
\label{sec:Introduction} 

Accurate localization is important for autonomous robots to perform complex tasks, especially in GPS-denied environments. 
Although localization algorithms~\cite{openvins,fast-lio2,vins,lio-sam} are reliable in most cases, scenarios still exist where Localization Accuracy and Stability (LAS) can get poor.
One solution is to make the robot actively choose trajectories conducive to improving LAS, known as Perception-aware Planning.

In general, a Perception-aware Planning framework pivots on two essential components: observation evaluation and trajectory generation.
By evaluating the accumulated loss of observations simulated along the candidate trajectories, the robots can choose one that is most conducive to improving LAS.
Recently, considerable progress~\cite{costante2016perception,zhang2018perception,zhang2019beyond,bartolomei2020perception,chen2024apace} has been made on vision-based Perception-aware Planning frameworks, yet related research based on LiDAR remains limited.
To improve the LiDAR-based LAS of robots, it's intuitive to adapt the design principles from vision-based frameworks.
However, the essential components require significant modifications due to the distinct attributes of cameras and LiDAR.

On the one hand, the focus of observation evaluation in the framework of the two kinds is different.
The main factors that affect the vision-based LAS are texture richness~\cite{costante2016perception} and lighting conditions~\cite{falanga2018pampc}, while for the LiDAR-based LAS, it is the geometric structure of the surroundings~\cite{loam}.

On the other hand, the efficiency of simulating visual feature points is much higher than that of simulating a LiDAR scan. The visual feature points are sparse enough to allow real-time simulation~\cite{costante2016perception}, even enabling the construction of visibility functions for each feature point~\cite{zhang2020fisher}.
However, LiDAR points are dense, making it impractical to simulate scans during motion planning.

Therefore, we need a new metric for LiDAR-based observation loss evaluation to produce reliable guidance in trajectory generation and an efficient framework that rationally separates observation evaluation from motion planning to prevent the time-consuming simulation from blocking the trajectory generation process.

\begin{figure}
    \centering
    \includegraphics[width=\columnwidth]{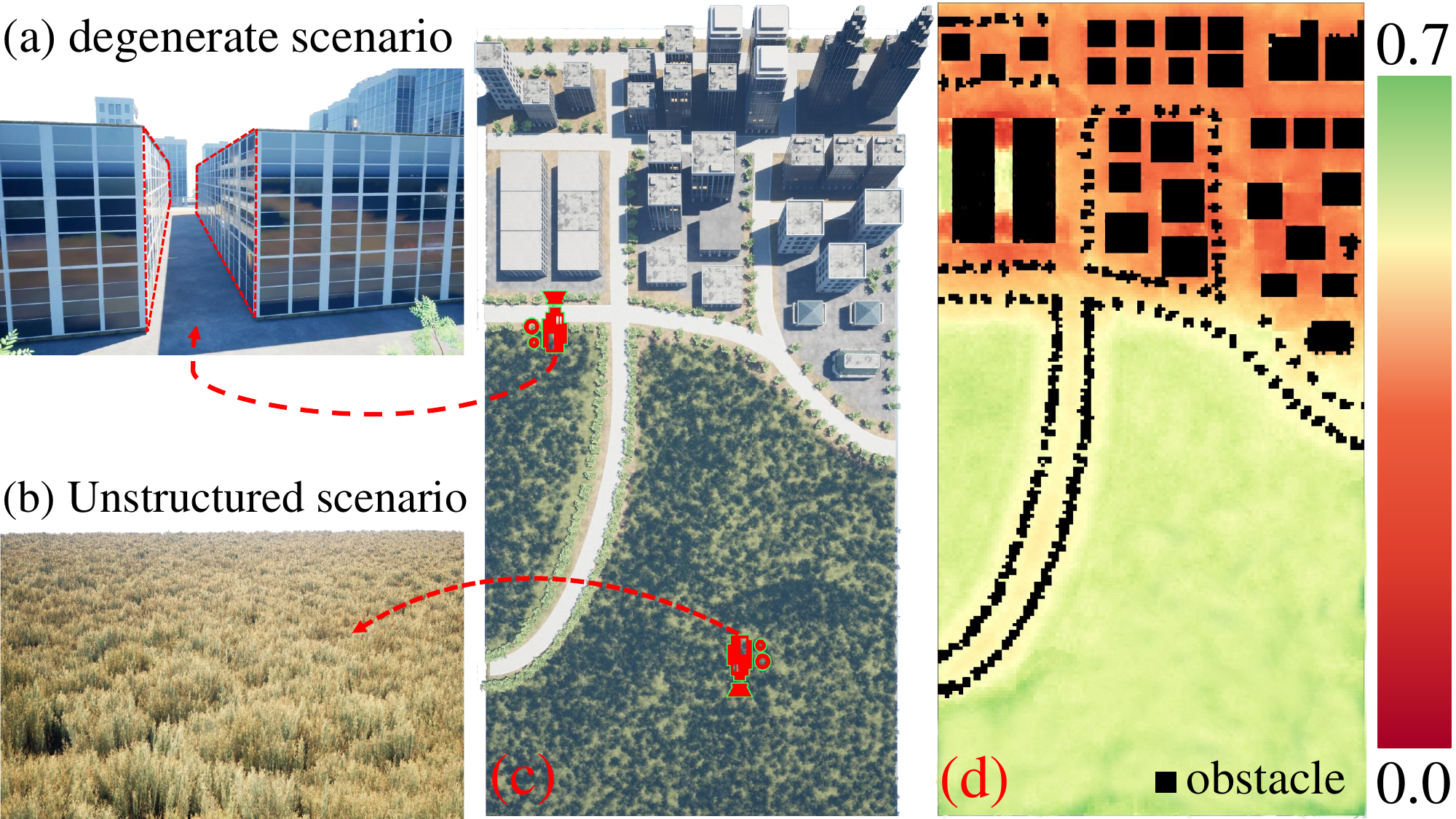}
    \caption{
    (a) Narrow corridor between two buildings.
    (b) Lush grasslands.
    (c) Bird's eye view of the entire map.
    (d) SOLM of the map with the observation model given by Eq.~(\ref{equ:observation}), where areas with smaller observation loss are preferable for localization and black areas indicate obstacles.}
    \label{fig:toutu}
    \vspace{-0.7cm}
\end{figure}

We note that the LiDAR-based LAS primarily depends on the geometric structure of the environment, which is inherently more stable than vision-based LAS due to its invariance to lighting conditions.
This stability permits the pre-evaluation of the loss of LiDAR observation at the robot's potential state space like $SE(2)$ or $SE(3)$.
In addition, since observations are roughly the same between two similar poses, the loss does not drastically change within a small area with a slight orientation difference.
Thus, we can discretize the state space into grids, evaluate the observation loss once at the center of each grid, and then use interpolation to obtain an evaluation of the intermediate states. 
We name the grid structure as Static Observation Loss Map (SOLM).

As mentioned earlier, we need to tailor a new metric to evaluate LiDAR observation loss at each grid in SOLM.
While covariance is a widely used metric~\cite{zhang2018perception} to quantify the observation loss by localization uncertainty,
it does not offer independent evaluation at one single location because the result depends on historical observations collected in other grids.
Additionally, some metrics like pre-defining the empirical principle~\cite{takemura2022perception} show insights into the localization process. 
However, lacking theoretical support may diminish the effectiveness of the metrics across diverse scenarios and with varying LiDAR configurations.

Therefore, a theoretically well-founded metric is required to evaluate LiDAR observation loss at a single location.
Given that Mainstream filtering-based~\cite{fast-lio2} or optimization-based~\cite{lio-sam} localization methods are essentially solving a Least Squares Problem, our analysis starts from the same form.
Inspired by the sensitivity analysis of linear Least Squares Problem~\cite{zhang2016degeneracy, salsp}, we propose a metric by introducing perturbations into LiDAR observations.

Fig.~\ref{fig:toutu} shows an example of SOLM. A ground robot with a rotating LiDAR localizes itself in the environment, as shown in Fig.~\ref{fig:toutu}(c).
Since the yaw angle does not affect observations with a 360-degree LiDAR, we calculate SOLM in $\mathbb{R}^2$ space to eliminate unnecessary computations. In Sec.~\ref{sec:Experiments}, a case with a limited field of view will be presented.
According to the SOLM in Fig.~\ref{fig:toutu}(d), degenerate corridors (Fig.~\ref{fig:toutu}(a)) and unstructured grasslands (Fig.~\ref{fig:toutu}(b)) are identified as less favorable for improving LAS, whereas places with rich structured features are deemed beneficial.

During motion planning, we consider SOLM alongside other factors, such as obstacle avoidance, trajectory smoothness, and agility, to generate trajectories favorable for localization.
We apply the proposed framework in more scenarios in Sec.\ref{sec:Experiments} to validate the effectiveness.
The experiments show that benefiting from the reliable metric and well-designed framework, the robot can improve its LAS by selecting topologies and orientations with low observation loss.

In summary, this paper makes the following contributions:
\begin{itemize}
\item [1)]We propose a novel perturbation-induced metric to evaluate LiDAR observation loss in SOLM calculation.
\item [2)]We design an efficient LiDAR-based Perception-aware Planning framework leveraging LiDAR attributes to decouple observation evaluation and motion planning.
\item [3)]We conduct extensive experiments to validate the effectiveness of the proposed framework. The entire project is open-sourced to promote further research in this field.
\end{itemize}

\section{Related Work}
\label{sec:RelatedWork}

In the realm of metric design for LiDAR-based Perception-aware Planning frameworks, empirical metrics~\cite{takemura2022perception} have been explored to evaluate LiDAR observation loss, offering insights into the localization process. However, these metrics are not well theoretically supported, which may lead to diminished effectiveness across various scenarios and different sensor configurations, highlighting the need for more universally applicable metrics.

Some researchers have proposed metrics~\cite{zhang2019beyond,zhang2016degeneracy,salaris2017online} that evaluate observation loss at a specific location by examining the observability of observations.
These metrics are adept at identifying degenerate scenarios, such as open fields or tunnels. Yet, degenerate scenarios are not the sole contributors to the diminished LAS. Factors like dynamic objects~\cite{ding2020lidar}, unstructured environments~\cite{guivant2004navigation}, and sensor noise~\cite{yuan2022efficient} interference also play significant roles.

Cognetti et al.~\cite{cognetti2018optimal} recognized it is insufficient to focus only on observability without considering noise and use pose estimation covariance as the metric. It works in vision-based frameworks, but as mentioned in Sec.\ref{sec:Introduction}, covariance is not an ideal metric to provide independent evaluation for SOLM calculation due to its historical observations. Without the help of SOLM, we have to simulate the LiDAR scan for covariance calculation during motion planning, which will reduce the efficiency of the entire framework.

Therefore, a theoretically founded and independent metric for evaluating LiDAR observation loss is necessary.
This paper proposes a novel metric via perturbation and sensitivity analysis~\cite{salsp}. In contrast to the approach in work~\cite{zhang2016degeneracy}, where perturbations are applied to additional constraints, we directly apply perturbations to the observations and analyze their effect on the solution. We find that the degeneracy factor derived in work~\cite{zhang2016degeneracy} is a special case of the proposed metric.

Regarding the framework for Perception-aware Planning, we focus on how existing strategies incorporate the observation loss into motion planning. 
Strategy~\cite{zhang2019beyond,zhang2020fisher} models observation loss as a differentiable function concerning the robot's pose. This approach assumes feature points are always visible, which may not hold in complex environments.
Another strategy~\cite{zhang2018perception} selects the final trajectory with the lowest covariance from several candidate trajectories that meet both obstacle avoidance and kinematic requirements.
Strategy~\cite{costante2016perception} uses covariance as a heuristic in exploring paths with Rapid Random Trees. 
The above three strategies work well in vision-based frameworks, while it is impractical for LiDAR-based Perception-aware Planning, primarily because vision feature points are sparse and can be simulated in real-time, whereas LiDAR point clouds are dense and simulating LiDAR scans is time-consuming, making trajectory generation has a long time span.

We decouple the evaluation process and motion planning into a two-phase pipeline to ensure the efficiency of motion planning. At first, the proposed perturbation-induced metric is employed to calculate SOLM, which is then stored as gray-scale images. When a new trajectory needs to be generated, the motion planning module will use the stored SOLM in the front- and back-end, finally optimizing a trajectory that can balance LAS, obstacle avoidance, and agility.

\section{Metric derivation}
\label{sec: Metric derivation}

In this section, we derive a new metric to evaluate LiDAR observation loss, facilitating SOLM computation.
In Sec.~\ref{subsec:Problem statement and method intuition}, we state the problem to be solved and give the intuition behind our method. In Sec.~\ref{subsec: Metric derivation}, we present an overview of the derivation process to outline our methodology and put details in the appendix for better readability.

\subsection{Problem statement and method intuition}
\label{subsec:Problem statement and method intuition}
Assume the robot's pose is denoted by $\boldsymbol{x}$, which has $n$ dimensions.
Data captured at pose $\boldsymbol{x}$ by sensors can be utilized to construct error-based observations, denoted as $\boldsymbol{h}(\boldsymbol{x})=[h_1(\boldsymbol{x}),h_2(\boldsymbol{x}),...,h_m(\boldsymbol{x})]^\text T$.
An example of $h_j(\boldsymbol{x})$ is given by Eq.(\ref{equ:observation}) in Sec.~\ref{sec:Experiments}, which means the distance from point $j$ to its nearest plane.
We aim to develop a metric to evaluate LiDAR observation loss that quantifies the effect of a LiDAR observation $\boldsymbol{h}(\boldsymbol{x})$ on the robot LAS.
The metric can be formulated as:

\begin{equation}
q=Q(\boldsymbol{h}(\boldsymbol{x})),\label{eq:metric}
\end{equation}
where $Q$ denotes the metric, and scalar $q$ is the observation loss. Regions with low $q$ values are more conducive to maintaining good LAS.

As mentioned in Sec.~\ref{sec:Introduction}, mainstream localization algorithms~\cite{fast-lio2,lio-sam} are essentially solving a Least Square Problem.
To derive a concrete form of $Q$, we explore how observations $\boldsymbol{h}(\boldsymbol{x})$ influence robots' LAS in the linear Least Square Problem context.
Consider the pose estimation problem as the following nonlinear Least Square Problem:
\begin{equation}
\boldsymbol x^*=\arg\mathop{\min}\limits_{\boldsymbol{x}} \frac{1}{2}\sum_{j=1}^{m}\lVert h_j(\boldsymbol{x})\rVert^2_2.\label{problem:nlsp}
\end{equation}
Usually, we solve Eq.(\ref{problem:nlsp}) iteratively with $\boldsymbol x_0$, an initial estimate of $\boldsymbol x$. Suppose we iterate one step using Gauss-Newton method\cite{optimization} to obtain $\boldsymbol{x}_\text{opt}$, which is equivalent to solving the following optimization problem:
\begin{align}
\boldsymbol{x}_{\text{opt}}=\arg\mathop{\min}\limits_{\boldsymbol{x}}\frac{1}{2}\lVert A(\boldsymbol x-\boldsymbol x_{0})-\boldsymbol b\rVert_2^2,\label{problem:llsp}
\end{align}
where $A\in\mathbb R^{m\times n},\boldsymbol b\in\mathbb R^m$. Each row of $A$ is $\nabla_{\boldsymbol{x}}^{\text T}h_j|_{\boldsymbol{x}=\boldsymbol{x}_0}$, each row of $\boldsymbol b$ is $-h_j(\boldsymbol{x}_0)$, and optimal pose estimation $\boldsymbol{x}_\text{opt}$ is directly affected by observations $\boldsymbol{h}(\boldsymbol{x})$.

In fact, the observations usually contain sensor noise, erroneous associations, etc., which can be regarded as a perturbation applied to the exact observations.
So we actually solve for $\hat{\boldsymbol{x}}_\text{opt}$ under disturbed observations $\hat{\boldsymbol{h}}(\boldsymbol{x})$. Naturally, we expect the displacement of the solution $\|\hat{\boldsymbol{x}}_\text{opt}-{\boldsymbol{x}}_\text{opt}\|$ caused by perturbations as small as possible. 

According to the linear Least Square Problem sensitivity theory~\cite{salsp}, given a certain disturbance on observations, the displacement of $\boldsymbol{x}_\text{opt}$ depends on $\boldsymbol{h}(\boldsymbol{x})$. In other words, the attributes of $\boldsymbol{h}(\boldsymbol{x})$ decide the robustness of the solution. 
Hence, if the metric $Q$ effectively captures the perturbation-induced displacement of $\boldsymbol{x}_\text{opt}$, it can serve as a reliable measurement of observation loss at the given pose $\boldsymbol{x}$. 

\subsection{Metric derivation}
\label{subsec: Metric derivation}
If we have a smooth enough perturbation mapping $\mathcal{V}_{\delta\boldsymbol{v}}$ and a perturbation $\delta\boldsymbol{v}\in\mathbb{R}^\text V$ that change the observations $\boldsymbol{h}$ to $\hat{\boldsymbol{h}}$, where $\mathcal{V}_{\delta\boldsymbol{v}}:\boldsymbol{h}(\boldsymbol{x})\mapsto\hat{\boldsymbol{h}}(\boldsymbol{x},\delta\boldsymbol{v})$, $\boldsymbol{x}_\text{opt}$ will be changed to $\hat{\boldsymbol{x}}_\text{opt}$.
Note that the perturbation mapping $\mathcal{V}_{\delta\boldsymbol{v}}$ with different expressions will lead to different changes in $\boldsymbol{x}_\text{opt}$. In this work, we assume that the mapping $\mathcal{V}_{\delta\boldsymbol{v}}$ has the form of a linear transformation with a bias, which can be expressed as the following equations:
\begin{align}
    \hat{\boldsymbol{h}}(\boldsymbol{x},\delta\boldsymbol{v})&=(I+\delta K)\boldsymbol{h}(\boldsymbol{x})+\delta\boldsymbol{t},\\
    \delta K&=g(\delta\boldsymbol{k})\in\mathbb R^{m\times m},\\
    \delta\boldsymbol{v} &= [\delta\boldsymbol{k}, \delta\boldsymbol{t}]^\text{T},
\end{align}
where $\delta\boldsymbol{t}\in\mathbb{R}^{m}, \text V=m^2+m$. The role of $g$ is to resize the vector $\delta\boldsymbol{k}\in\mathbb{R}^{m^{2}}$ to the matrix $\delta K$ in row-major order~\cite{row-major}.
Substituting $\hat{\boldsymbol{h}}(\boldsymbol{x})$ into Problem~(\ref{problem:llsp}), $\hat{\boldsymbol{x}}_\text{opt}$ can be obtained by solving the following optimization problem:
\begin{align}
\hat{\boldsymbol{x}}_\text{opt}=\arg\mathop{\min}\limits_{\boldsymbol{x}}\frac{1}{2}\lVert& (I+\delta K)A(\boldsymbol x-\boldsymbol x_{0})\nonumber\\
&-(\boldsymbol b-\delta K\boldsymbol{b}-\delta\boldsymbol{t})\rVert_2^2.
\end{align}
Let $\Delta\boldsymbol{x}=\boldsymbol{x}-\boldsymbol{x}_0$, we have:
\begin{align}
\boldsymbol{x}_\text{opt}&=\arg\mathop{\min}\limits_{\Delta\boldsymbol{x}}\frac{1}{2}\lVert A \Delta\boldsymbol{x}-\boldsymbol b\rVert^2_2+\boldsymbol{x}_0,\\
\hat{\boldsymbol{x}}_\text{opt}&=\arg\mathop{\min}\limits_{\Delta\boldsymbol{x}}\frac{1}{2}\lVert (A+\Delta A) \Delta\boldsymbol{x}-(\boldsymbol b-\Delta \boldsymbol b)\rVert^2_2+\boldsymbol{x}_0,
\end{align}
where $\Delta A=g(\delta\boldsymbol{k})A$ and $\Delta \boldsymbol b=g(\delta\boldsymbol{k})\boldsymbol{b}+\delta\boldsymbol{t}.$

What we care is how much small $\delta\boldsymbol k$ and $\delta\boldsymbol t$ can affect the displacement norm of the solution $\|\hat{\boldsymbol{x}}_\text{opt}-{\boldsymbol{x}}_\text{opt}\|$. We find that this problem can be converted to a sensitivity analysis problem for the optimal solution of a linear Least Square Problem. Work\cite{salsp} summarizes a number of methods for obtaining the upper bound on the variation of the solution caused by $\Delta A$ and $\Delta\boldsymbol b$. Let the expression for the upper bound be $E(\lVert\Delta A\rVert,\lVert\Delta\boldsymbol b\rVert)$, we have
\begin{equation}
\lVert\hat{\boldsymbol{x}}_\text{opt}-\boldsymbol{x}_\text{opt}\rVert
\leq E(\lVert\Delta A\rVert,\lVert\Delta\boldsymbol b\rVert).\label{eq:err_bound}
\end{equation}
Writing the perturbations $\delta \boldsymbol k$ and $\delta\boldsymbol t$ as: $\delta\boldsymbol k=r_{\boldsymbol k}\boldsymbol \alpha,\delta\boldsymbol t=r_{\boldsymbol t}\boldsymbol \beta$, where $\boldsymbol \alpha^\text T\boldsymbol \alpha=\boldsymbol \beta^\text T\boldsymbol \beta=1,r_{\boldsymbol k}\ge0,r_{\boldsymbol t}\ge0$, and substituting them into the upper bound $E$, we have: 
\begin{align}
E(\lVert\Delta A\rVert,\lVert\Delta\boldsymbol b\rVert)&=E(\lVert g(\delta\boldsymbol{k})A\rVert, \lVert g(\delta\boldsymbol{k})\boldsymbol{b}+\delta\boldsymbol{t}\rVert)\nonumber\\
&\triangleq F(r_{\boldsymbol k}\boldsymbol{\alpha},r_{\boldsymbol t}\boldsymbol{\beta}).
\end{align}
Treating $\boldsymbol{\alpha},\boldsymbol{\beta}$ as constants first and linearizing $F$ near $r_{\boldsymbol k}=r_{\boldsymbol t}=0$ (which means the intensities of the perturbations tend to zero), we have
\begin{align}
F(r_{\boldsymbol k}\boldsymbol{\alpha},r_{\boldsymbol t}\boldsymbol{\beta}) 
&\approx\left.\frac{\partial F}{\partial r_{\boldsymbol k}}\right|_{r_{\boldsymbol k}=r_{\boldsymbol t}=0}r_{\boldsymbol k}+ \left.\frac{\partial F}{\partial r_{\boldsymbol t}}\right|_{r_{\boldsymbol k}=r_{\boldsymbol t}=0}r_{\boldsymbol t}\nonumber\\
&\triangleq d_1(\boldsymbol{\alpha,\beta})r_{\boldsymbol k}+d_2(\boldsymbol{\alpha},\boldsymbol{\beta})r_{\boldsymbol t}.
\end{align}
We note that $d_1$ and $d_2$ serve as an amplifier of perturbations, which, in other words, denotes the sensitivity of $F$ to the perturbations.
As the upper bound of the displacement of the solution $\|\hat{\boldsymbol{x}}_\text{opt}-{\boldsymbol{x}}_\text{opt}\|$, we expect $F$ less sensitive to the perturbations, which will constrain the sensitivity of displacement to the perturbations at a level no more than the sensitivity of $F$, because the inequality sign still holds for the derivation of both sides of Inequality (\ref{eq:err_bound}) at zero. This can be easily proved by considering the functions on the left and right sides of Inequality (\ref{eq:err_bound}) as $f_1$ and $f_2$ in Lemma~\ref{app: A}, respectively. 
Thus, we define the metric $Q$ as the sensitivity of the upper bound $F$ to the perturbations:
\begin{equation}
   q=\sqrt{w_1d_1^2(\boldsymbol{\alpha}_\mathcal{G},\boldsymbol{\beta}_\mathcal{G})+w_2d_2^2(\boldsymbol{\alpha}_\mathcal{G}, \boldsymbol{\beta}_\mathcal{G})},\label{eq:qo}
\end{equation}
where $\boldsymbol{\alpha}_\mathcal{G}, \boldsymbol{\beta}_\mathcal{G}$ are user-chosen directions of perturbations, $w_1>0,w_2>0,w_1+w_2=1$ are the weights of $d_1, d_2$, respectively. Here, instead of performing sensitivity analysis on the derivation of $\lVert \hat{\boldsymbol{x}}_\text{opt}-\boldsymbol{x}_\text{opt}\rVert$ directly, we perform sensitivity analysis on the upper bound $E$ because the expanded expression for $\lVert \hat{\boldsymbol{x}}_\text{opt}-\boldsymbol{x}_\text{opt}\rVert$ is too complicated to be analyzed.
After simplifying, what we need to do is pursue observations $\boldsymbol{h}(\boldsymbol{x})$ with lower $q$ value to keep low sensitivity of displacement  $\|\hat{\boldsymbol{x}}_\text{opt}-{\boldsymbol{x}}_\text{opt}\|$ to perturbations, which is conducive to improving the LAS.

\begin{figure}
    \centering
    \subcaptionbox{$\lVert\partial F/\partial\delta\boldsymbol{t}\rVert_2^2$ when $\delta\boldsymbol{t}$ is near zero}    {\includegraphics[width=1.0\columnwidth]{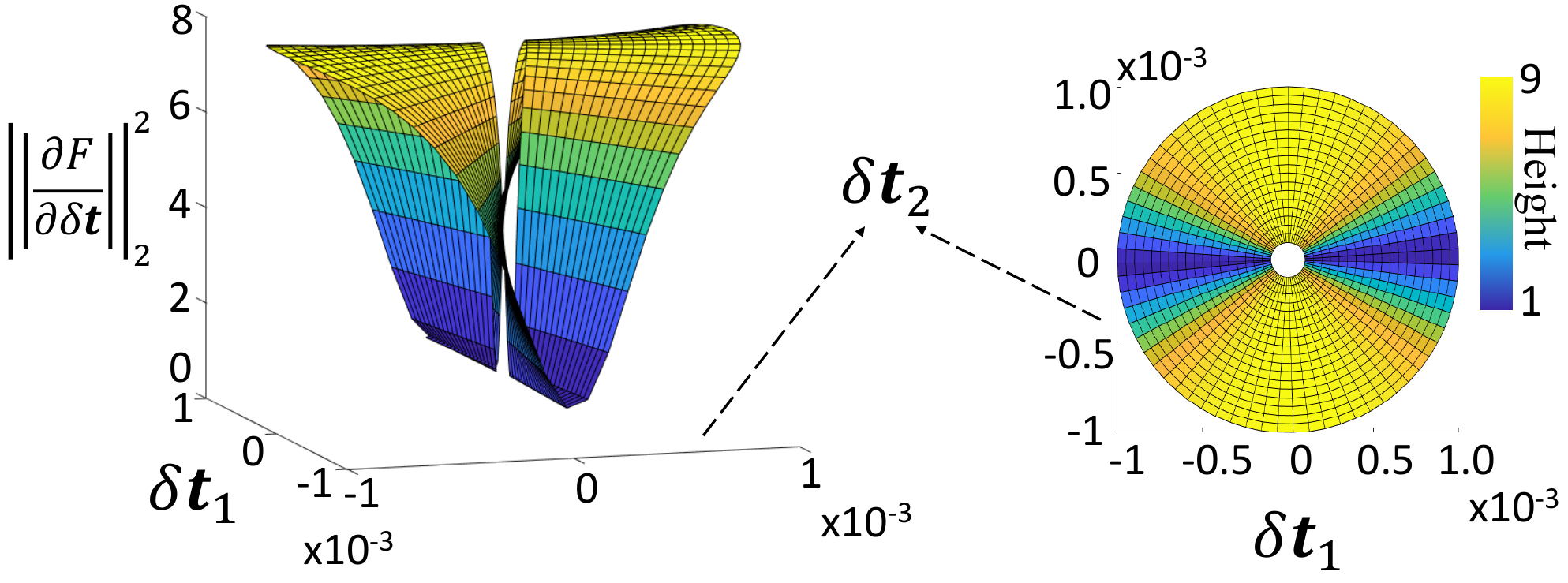}}
    \subcaptionbox{$\boldsymbol{\beta}-\lvert q_o\rvert$ curve in polar coordinates}{\includegraphics[width=1.0\columnwidth]{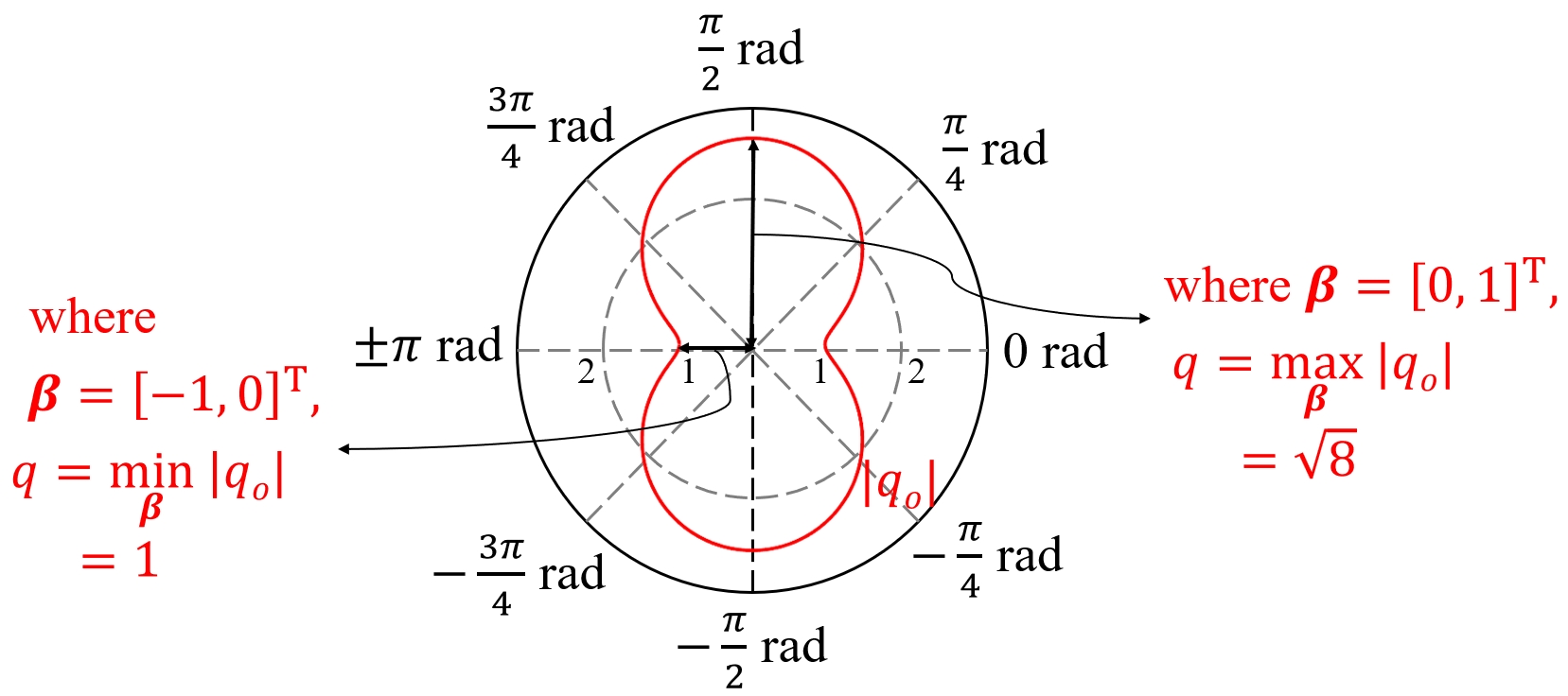}}
    \caption{A special example for illustrating the proposed metric, where $E(\lVert\Delta A\rVert,\lVert\Delta\boldsymbol b\rVert)=\sqrt{\delta\boldsymbol{t}_1^2+8\delta\boldsymbol{t}_2^2},\ \delta\boldsymbol{t}=[\delta\boldsymbol{t}_1,\delta\boldsymbol{t}_2]^\text T\in\mathbb{R}^2$. Figure (a) illustrates the results of directly solving for $F$ on the derivative of $\delta\boldsymbol{t}$, where $\lVert\partial F/\partial\delta\boldsymbol{t}\rVert_2^2=(\delta\boldsymbol{t}_1^2+64\delta\boldsymbol{t}_2^2)/(\delta\boldsymbol{t}_1^2+8\delta\boldsymbol{t}_2^2)$. The figure to the right is a top view of the surface on the left. Since the derivative of $F$ does not exist at zero, there is a hole in the surface. Figure (b) shows how the proposed metric looks like in the polar coordinate system $\theta-r$, where $\boldsymbol{\beta}=[\cos\theta,\sin\theta]^\text T,\ r=\lvert q_o\rvert=\lvert d_2\rvert=\sqrt{\cos\theta^2+8\sin\theta^2}$. We can clearly see the effect of perturbations with different directions on the sensitivity of the upper bound $E$ from this figure.\label{bbb}}
    \label{fig:eg_metric}
    \vspace{-0.2cm}
\end{figure}

The intuition for defining the metric as Eq.(\ref{eq:qo}) is that the directions of the perturbations can affect the sensitivity of the upper bound $E$.  From a mathematical point of view, it is impossible to obtain the sensitivity of $E$ by direct derivation of $\delta\boldsymbol{k}$ and $\delta\boldsymbol{t}$ without treating $\boldsymbol{\alpha},\boldsymbol{\beta}$ as constants first, since the derivatives of the norms usually do not exist at zero. To enhance the clarity of the abstract mathematical procedure, we present an example when $m=2,\ w_2=1$, visualizing $\lVert\partial F/\partial\delta\boldsymbol{t}\rVert_2^2$ and $\lvert q_o\rvert=\sqrt{w_1d_1^2+w_2d_2^2}$ in Fig. \ref{fig:eg_metric}.

In this work, we choose one of the upper bounds derived in the work \cite{lawson1995} to compute the metric. According to the results of the work\cite{lawson1995}, suppose that $A$ has full column rank and $\lVert\Delta A\rVert_2<\sigma_1$, where $\sigma_1$ is the smallest singular value of $A$, we have
\begin{align}
E(\lVert\Delta A\rVert,\lVert\Delta\boldsymbol b\rVert)&=\frac{\lVert\Delta\boldsymbol{x}^*\rVert_2+\sigma_1^{-1}\lVert\boldsymbol{r}\rVert_2}{\sigma_1-\lVert\Delta A\rVert_2}\lVert \Delta A\rVert_F\nonumber\\
&\quad+\frac{1}{\sigma_1-\lVert\Delta A\rVert_2}\lVert \Delta\boldsymbol{b}\rVert_2,\label{eq:sp_errb}
\end{align}
where $\Delta\boldsymbol{x}^*=\arg\mathop{\min}\limits_{\Delta\boldsymbol{x}}\frac{1}{2}\lVert A \Delta\boldsymbol{x}-\boldsymbol b\rVert^2_2,\boldsymbol{r}=A\Delta\boldsymbol{x}^*-\boldsymbol{b}$.

Writing $\boldsymbol{\alpha}$ as $\boldsymbol{\alpha}=[\delta\boldsymbol{k}_1^\text T,\delta\boldsymbol{k}_2^\text T,...,\delta\boldsymbol{k}_m^\text T]^\text T$ and calculating the corresponding derivatives (Details can be found in Appendix.~\ref{app: B}), we can obtain:

\begin{equation}
q(\boldsymbol{\alpha})=\frac{1}{\sigma_1}\sqrt{w_1\sum_{j=1}^m\delta\boldsymbol{k}_j^\text T\Phi\delta\boldsymbol{k}_j+w_2},\label{eq:sp_metric}
\end{equation}
where $\Phi=\xi^2AA^\text T+\boldsymbol{b}\boldsymbol{b}^\text T,\xi=\lVert\Delta\boldsymbol{x}^*\rVert_2+\lVert\boldsymbol{r}\rVert_2/\sigma_1$.

Noting that $\Phi$ is a symmetric semi-positive definite matrix and $m\ge n$, we can diagonalize it: $\Phi=P\Lambda P^\text T$, where $PP^\text T=I,\Lambda=\text{diag}(\lambda_m,\lambda_{m-1},...,\lambda_1)$, $0=\lambda_1=\lambda_2=...=\lambda_{m-n-1}\le\lambda_{m-n}\le...\le\lambda_m$ are eigenvalues of $\Phi$. Since $\boldsymbol{\alpha}^\text T\boldsymbol{\alpha}=1$, Eq.(\ref{eq:sp_metric}) is equivalent to:
\begin{equation}
q(\boldsymbol{\alpha})=\frac{1}{\sigma_1}\sqrt{w_1\sum_{j=1}^m\delta\boldsymbol{k}_j^\text T\Lambda\delta\boldsymbol{k}_j+w_2}.
\end{equation}

As illustrated in Fig. \ref{fig:eg_metric}$(b)$, the evaluation is affected by the direction of perturbations. With eigenvalue analysis, we can obtain the range of evaluation result:  $\sum_{j=1}^m\delta\boldsymbol{k}_{j}^\text T\Lambda\delta\boldsymbol{k}_{j} \ge \sum_{j=1}^m\lambda_1\delta\boldsymbol{k}_{j}^\text T\delta\boldsymbol{k}_{j}=\lambda_1$
and 
$\sum_{j=1}^m\delta\boldsymbol{k}_j^\text T\Lambda\delta\boldsymbol{k}_j\le \sum_{j=1}^m\lambda_m\delta\boldsymbol{k}_j^\text T\delta\boldsymbol{k}_j=\lambda_m$.

Thus, choosing different value of $\sum_{j=1}^m\delta\boldsymbol{k}_{j}^\text T\Lambda\delta\boldsymbol{k}_{j}$, we can obtain different forms of $q$. The three most representative forms are:

\begin{equation}
    q^{(\min)}=\frac{1}{\sigma_1}\sqrt{w_1\lambda_1+w_2}=\frac{\sqrt{w_2}}{\sigma_1}.\label{equ:lambda_min}
\end{equation}
\begin{equation}
    q^{(n)}=\frac{1}{\sigma_1}\sqrt{w_1\lambda_{m-n+1}+w_2},\label{equ:lambda_n}
\end{equation}
\begin{equation}
    q^{(\max)}=\frac{1}{\sigma_1}\sqrt{w_1\lambda_m+w_2},\label{equ:lambda_max}
\end{equation}

Different strategies actually represent different conservativeness. $q^{(\max)}$ is the most conservative estimate, choosing the direction in which the perturbation has the greatest impact. On the contrary, $q^{(\min)}$ is the least conservative strategy, which is actually the degradation factor in literature~\cite{zhang2016degeneracy}. $q^{(n)}$ is an intermediate strategy between $q^{(\max)}$ and $q^{(\min)}$, where $n$ denotes the dimension of the state $\boldsymbol{x}$ to be estimated. In Sec.~\ref{sec:Experiments}, we will show the effect of strategies with different levels of conservativeness in observation evaluation and trajectory generation.

\begin{figure*}[h]
    \centering
    \includegraphics[width=\textwidth]{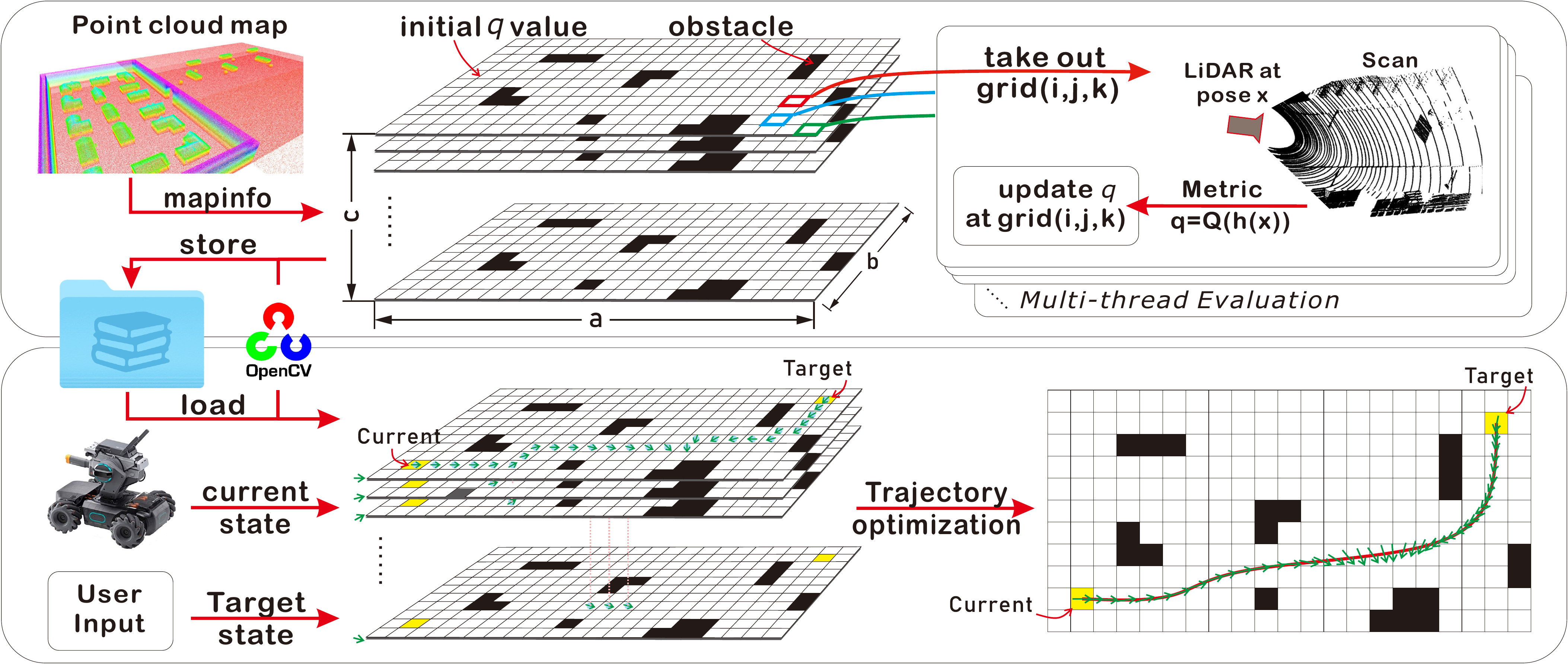}
    \caption{The proposed perception-aware planning framework. The \textbf{top} part represents the SOLM calculation process, while the \textbf{bottom} part represents the motion planning process.}
    \label{fig:pipeline}
    \vspace{-0.4cm}
\end{figure*}

\section{Framework design}
\label{sec: Framework design}
In this section, we show the framework of the proposed LiDAR-based Perception-aware Planning,
including SOLM calculation with the metrics derived in Sec.~\ref{sec: Metric derivation}, trajectory planning based on SOLM, and the relation of these two parts.

\subsection{SOLM Calculation}
\label{subsec:SOLM Calculation}
The upper part of Fig.~\ref{fig:pipeline} shows the process of SOLM calculation, including initialization, updating, and saving.

To start with, we initialize the SOLM according to the robot's potential state space. Assume the robot's state is $\boldsymbol{x}= [x,y,\theta]^\text T\in SE(2)$. Given the resolution on each dimension, we discretize the $SE(2)$ state space a grid structure with the size of $a,b$, and $c$ for dimensions $x$, $y$, and $\theta$, respectively.
The obstacles within the given point cloud map will be marked on the SOLM, where the observation evaluation will be skipped to reduce unnecessary calculations.

Then, we evaluate the observation loss at each grid. Firstly, the information of a non-obstacle $\text{grid}(i,j,k)$ is taken out from the SOLM.
Secondly, we simulate a LiDAR scan at the state represented by the center of this grid.
A ray-casting-based scan simulator with GPU acceleration can significantly speed up the simulation process, like the one used in work~\cite{kong2023marsim}.
The simulated scan will formulate observations $\boldsymbol{h}(\boldsymbol{x})$ follow the observation model like Eq.~(\ref{equ:observation}).
Thirdly, we evaluate the observation loss at state $\boldsymbol{x}$ with the metric $Q$ derived in Sec.~\ref{sec: Metric derivation}. Finally, we update the evaluation result $q$ into $\text{grid}(i,j,k)$.
It is noted that since the evaluation process at different grids is independent of each other, the process can be performed in parallel.
Once all the non-obstacle grids have been evaluated, we obtain the SOLM that can be used for motion planning. 
In practice, due to the similar structure of the picture, we store the SOLM as multi-channel images for the $SE(2)$ case. For higher-dimensional state spaces, like $SE(3)$, we store the SOLM as a tensor.

\subsection{Motion planning}
\label{subsec:Motion planning}

The lower part of Fig.~\ref{fig:pipeline} shows the process of motion planning, including path searching and trajectory optimization.

After loading SOLM from stored images, we restore the knowledge of observation loss in the robot's state space.
For front-end path searching, we adopt the SOLM as the cost function and perform Breadth-First-Search (BFS) in the grid map, which allows for full consideration of all topologies and orientations in the environment to emphasize the effectiveness in improving the robot's LAS of the proposed Perception-aware Planning framework.

For trajectory optimization, similar to the work\cite{uneven}, we use piecewise polynomials $\boldsymbol{x}(t)=[x(t),y(t),\theta(t)]^\text T$ to represent state trajectories, where $[x,y]\in\mathbb{R}^2$ denotes the position of the robot in the world frame, $\theta\in SO(2)$ denotes the yaw angle of the robot, as mentioned in last subsection. In this work, the cubic spline is chosen to represent the trajectory for the smoothness. We formulate the perception-aware trajectory planning problem for ground robots as the following optimization problem:
\begin{align}
&\min_{\boldsymbol{c},\boldsymbol{T}}\int_0^{T_s}q(\boldsymbol{x}(t))dt+\rho_TT_s\label{problem:opt}\\
&s.t.\ \ \boldsymbol{M}\boldsymbol{c}=\boldsymbol{b}, \quad \boldsymbol{T}\ge\boldsymbol{0},\label{con:minco}\\
&\quad \quad \text{Nonholo}(\boldsymbol{x})=0,\label{con:nonholo}\\
&\quad\quad \dot{\boldsymbol{x}}^2-([v_{\text{mlon}},v_{\text{mlat}},w_{\text{max}}]^\text T)^2\leq[0,0,0]^\text T,\label{con:dyn}\\
&\quad \quad \text{SDF}(\boldsymbol{x})\ge r_\text{safe},\label{con:esdf}
\end{align}
where $\boldsymbol{c}\in\mathbb{R}^{4\text N\times3}$ denote coefficient matrix, $\text N$ denotes the number of segments of the piecewise polynomials. Each element in $\boldsymbol{T}=[T_1,T_2,...,T_\text N]^\text T\in\mathbb{R}^\text N$ denotes the duration of a piece of the trajectory. In the object function, $q(\boldsymbol{x}(t))$ denote observation loss at state $\boldsymbol{x}(t)$, $T_s=T_1+T_2+...+T_\text N$ is the total duration of the trajectory, $\rho_T$ is a positive weight to ensure the aggressiveness of the trajectory. 

Conditions (\ref{con:minco}) denote the continuity constraints of the cubic spline at interpolation points and the positive time, where the inequality is taken element-wise. Conditions (\ref{con:nonholo}) denote the possible nonholonomic constraints of the robot. For car-like robots or differential robots, it is $\dot{\boldsymbol{x}}\sin\theta-\dot{\boldsymbol{y}}\cos\theta=0$. While for omnidirectional robots, this constraint does not exist. Conditions (\ref{con:dyn}) are dynamic feasibility constraints, where the inequality and square operation are also taken element-wise. $v_{\text{mlon}},v_{\text{mlat}},w_{\text{max}}$ denote maximum longitude velocity, latitude velocity, and angular velocity, respectively. Condition (\ref{con:esdf}) denotes safety constraint, where $\text{SDF}(\cdot)$ denotes the signed distance function (SDF) at position $[x,y]^\text T$ to the nearest obstacle. $r_\text{safe}$ denotes the safety distance related to the size of the robot. In this work, we compute the SDF by the algorithm proposed in work\cite{esdf}.

To make the above optimization problem easier to solve, we adopt the MINCO trajectory class and differential homogeneous mapping used in work\cite{minco} to eliminate the constraints (\ref{con:minco}). For ease of solving the problem, we also replace the integral in the objective function with a summation approximation. When the equation constraint (\ref{con:nonholo}) exists, PHR Augmented Lagrange Multiplier method\cite{alm} (PHR-ALM) is used to solve the simplified problem for higher accuracy. Conversely, we utilize the penalty function method, which is sufficient to obtain an acceptable solution. Besides, an efficient quasi-Newton method L-BFGS\cite{lbfgs} is chosen as the unconstrained optimization algorithm to work with PHR-ALM and the penalty function method.

\section{Experiments}
\label{sec:Experiments}

\subsection{Implementation details}
\label{subsec:Implementation details}
To validate the effectiveness and efficiency, we apply the proposed Perception-aware Planning framework on an omnidirectional vehicle (RoboMaster AI-2020) with a 70-degree Field of View (FoV) LiDAR (Livox Mid-70\footnote{\url{https://www.livoxtech.com/mid-70}}).
In real-world experiments, the odometry is provided by FAST-LIO2~\cite{fast-lio2}, a high-precision Lidar-Inertial Odometry (LIO) system. The ground truth of robot pose is provided by the motion tracking system (NOKOV\footnote{\url{https://www.nokov.com/}}).
All the computation is performed on an on-board computer with Intel i7-1165G7.
The observation model takes the distance from the scan point to the nearest plane, the same model defined in work~\cite{fast-lio2}:
\begin{align}
    h_j({^\text L}{\boldsymbol{p}_j},\boldsymbol{x}) = {^\text G}{\boldsymbol{u}_j}^\text T({^\text G}{\boldsymbol{R}_{L}}{^\text L}\boldsymbol{p}_j+{^\text G}{\boldsymbol{t}_{L}}-\boldsymbol{q}_j),\label{equ:observation}
\end{align}
where $\boldsymbol{x}$ indicates the robot position ${^\text G}\boldsymbol{t}_{\text L}$ and orientation ${^\text G}\boldsymbol{R}_{\text L}$ to be estimated, $\boldsymbol{p}_j$ is a LiDAR point, $\boldsymbol{u}_j$ is the normal vector of the nearest plane, $\boldsymbol{q}_j$ is a point on the nearest plane. 

To see if the metric is reliable in evaluating observation loss, we adopt an indicator Mean Disturbance-induced Error (MDE) serving as the evaluation ground truth, which is similar to $\boldsymbol{e}_p$ defined in work~\cite{nubert2022learning} as follows:
\begin{align}
    \text{MDE}=\frac{1}{N}\sum_{j=1}^{N}\rVert \log(\boldsymbol{T}^{-1}_{\text{gt}}\boldsymbol{T}_{\text{regi},j})^\vee\lVert_2^2,\label{equ:MDE}
\end{align}
where $N$ denotes the number of disturbances, $\boldsymbol{T}_{\text{regi},j}$ is the $j_{th}$ registration result start from a disturbed $\boldsymbol{T}_{\text{gt}}$ with iterated point-to-plane registration.

\subsection{Metric Effectiveness Validation}
\label{subsec:Metric Effectiveness Validation}
In this subsection, we demonstrate the effectiveness and characteristics of the proposed metrics in evaluating observation loss in several representative scenes built in UE5\footnote{\url{https://www.ue5.com}}.

As shown in Fig.~\ref{fig:scenarios}, scene (a) is rich in planar features where the robot can localize itself stably and robustly to observation perturbations. Scene (d) is a typical degenerate scenario.
Scene (c) is an unstructured scenario with roughly the same features, where the robot's pose estimation is sensitive to the initial guess $\boldsymbol{x}_0$.
Scene (b) shows some structured features in the unstructured grasslands.
\vspace{-0.3cm}

\begin{figure}[h]
    \centering
    \includegraphics[width=\columnwidth]{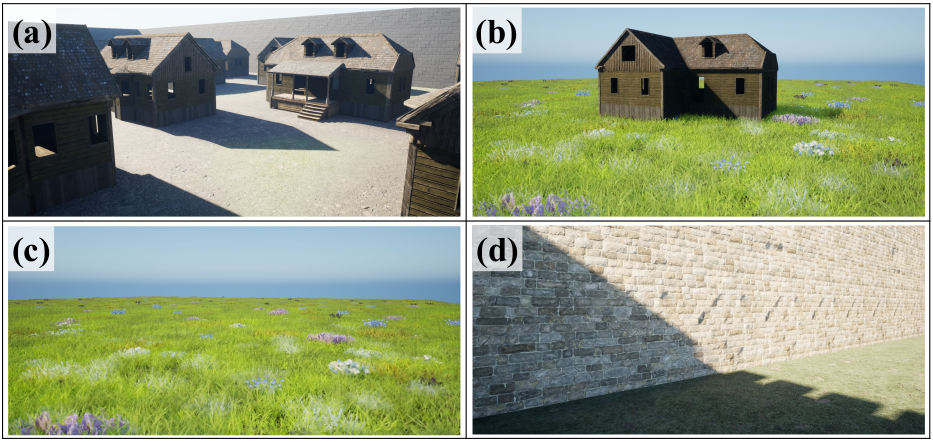}
    \caption{Four representative seances: (a) Houses, (b) Single house on the meadow, (c) Meadow, (d) Wall.}
    \label{fig:scenarios}
\end{figure}
\vspace{-0.3cm}

We first simulate a LiDAR scan at each scene to obtain observations $\boldsymbol{h}(\boldsymbol{x})$ with the LiDAR model described in Eq.~(\ref{equ:observation}).
Then we use metrics Eq.~(\ref{equ:lambda_min})-(\ref{equ:lambda_max}) to evaluate the loss of the simulated observation.
Besides, we calculate MDE for each scene as the evaluation's ground truth.
The result is shown in TABLE.~\ref{tab:bk}.
By comparing the ground truth with the evaluations by different metrics, we can judge which metric can provide the most reliable evaluation.
It's important to note that we are only concerned with the relative loss of different scenarios by the same metric. 

 \begin{table}[h]
    \vspace{-0.2cm}
	\small
	\centering
	\renewcommand\arraystretch{1.2}
	\caption{\label{tab:bk} observation loss in four scenes}
	\begin{tabular}{c|clllllllll}
		\hline
		\multirow{1}{*}{Metric}& \multirow{1}{*}{$a$} & \multicolumn{1}{c}{$b$} &  \multicolumn{1}{c}{$c$}&  \multicolumn{1}{c}{$d$}\\\hline
        $q^{(\text{min})}$ &0.320 &0.28 &0.31&12.32\\
        $q^{(\text{n})}$ &0.547 &0.91 &1.20&12.34\\
        $q^{(\text{max})}$ &12.15 &17.2 &16.3&568.8\\ 
        $\text{MDE}$ &0.733 &2.81 &3.80 &4.63\\
        \hline
	\end{tabular}
\end{table}

In Table.~\ref{tab:bk}, all of the three metrics can tell the observation loss in scene (d) is the worst, while only the evaluation with metric $q^{(n)}$ is consistent with the ground truth in the four scenes.
For $q^{(\text{min})}$, which is actually the degeneracy factor proposed in literature~\cite{zhang2016degeneracy}, cannot distinguish scenarios ($a$), ($b$), and ($c$) well because this metric only focuses on the structural richness of scenarios.
In addition, the reason why the evaluation with $q^{(\text{max})}$ does not distinguish scenarios ($b$) and ($c$) well is the introduction of over-conservative maximum eigenvalues, which leads to the evaluation influenced by the best-observed dimension of the state to be estimated.

\subsection{Experiment in Large Scale Map}
\label{subsec:Experiment in Large Scale Map}

To demonstrate the effectiveness of the proposed framework in scenarios with large ranges, we build a Gobi desert in UE5 as shown in Fig.~\ref{fig:sim}. We require the robot with the configuration in Sec.~\ref{subsec:Implementation details} to reach a distant warehouse from the inside of a small town. Stones are evenly distributed on the ground outside of towns to help robots locate themselves while traveling, and part of the stones are surrounded by weeds and bushes.
To further illustrate the characteristics of the framework with different metrics, we show two trajectories guided by metrics $q^{(\text{min})}$ and $q^{(n)}$ respectively.
\vspace{-0.5cm}

\begin{figure}[h]
    \centering
    \includegraphics[width=\columnwidth]{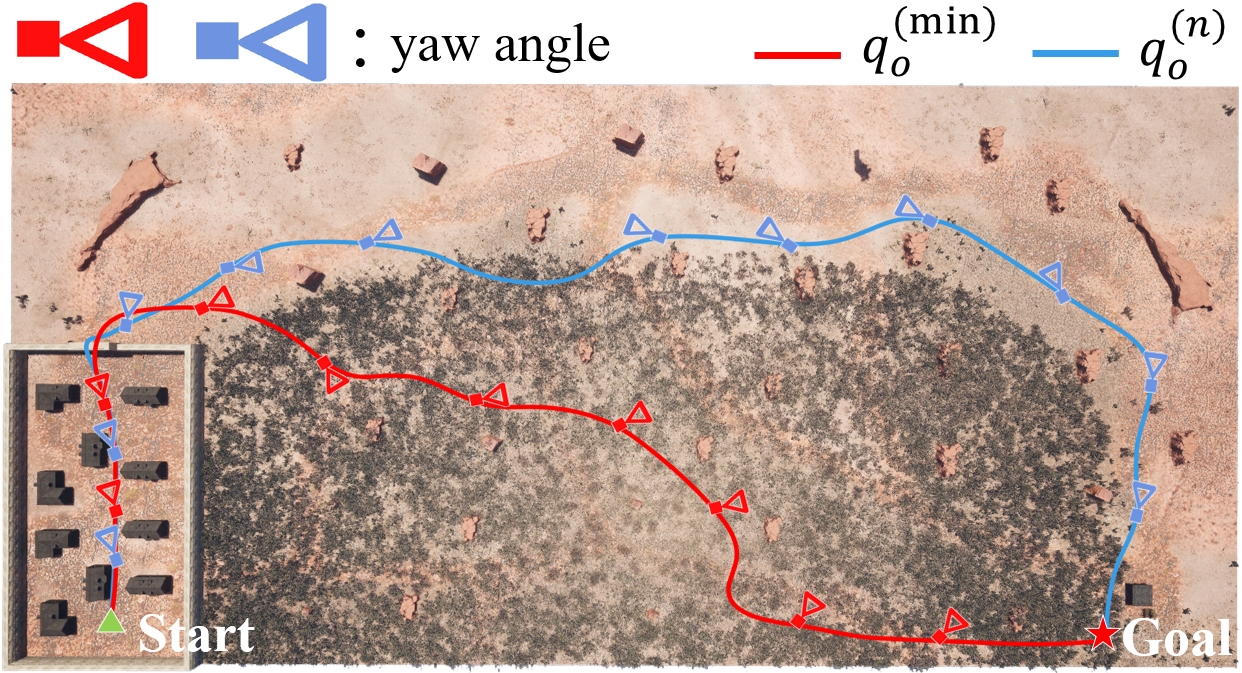}
    \caption{Robot trajectories by the proposed framework.}\label{fig:sim}
\end{figure}
\vspace{-0.4cm}

As we can see in Fig.~\ref{fig:sim}, the planner utilizing $q^{(\text{min})}$ chooses a trajectory with grass and bush, whereas the planner utilizing $q^{(\text{n})}$ prefers the path with well-structured features and devoid of grass and bush, which is because $q^{(\text{min})}$ only considers the feature richness of observations and make the robot pursuit area with rich features no matter whether it is preferable for robot localization. 
\vspace{-0.1cm}

\begin{figure}[h]
    \centering
    \includegraphics[width=\columnwidth]{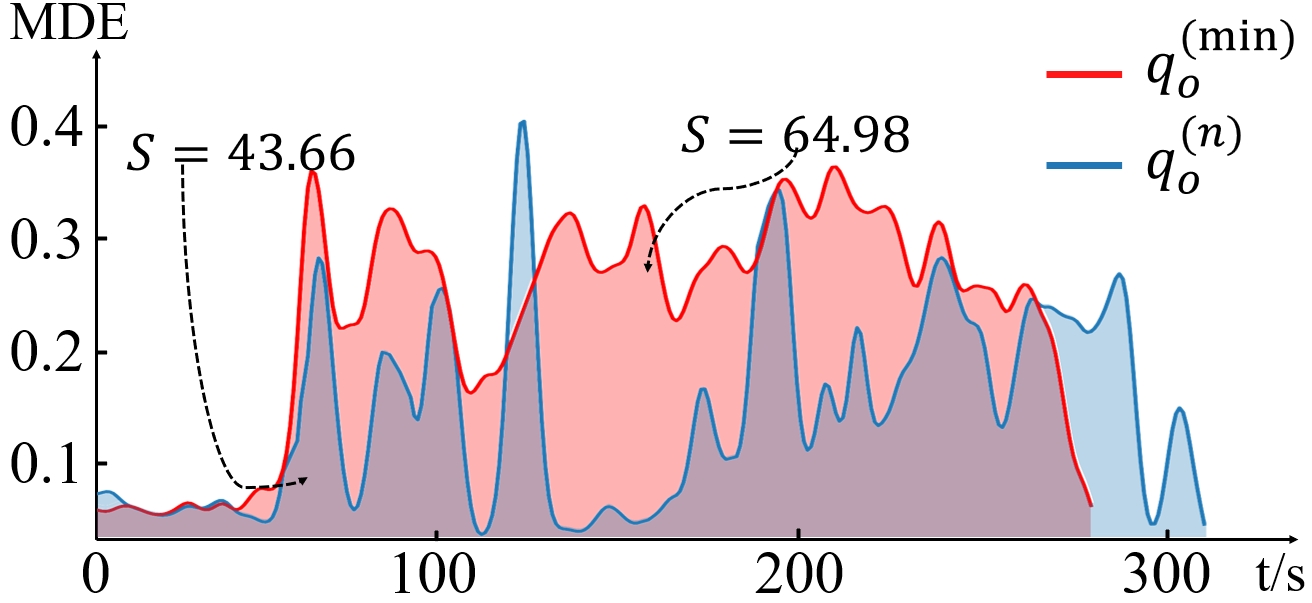}
	\caption{MDE along two different trajectories.}\label{fig:odom_err}
\end{figure}
\vspace{-0.3cm}

To evaluate how much effect the observations along the trajectories have on LAS, we calculate MDE for each pose sampled along the trajectories at a constant time interval. As shown in Fig.~\ref{fig:odom_err}, while the trajectory guided by $q^{(n)}$ is less rich in features compared to $q^{(\text{min})}$, the clear geometrical features result in smaller accumulated MDE ($S=43.66$) compared to the trajectory guided by $q^{(\text{min})}$ ($S=64.98$).

\subsection{Real-world Experiment}
\label{subsec:Real-world Experiment}
After showing the different features of the proposed metrics, we adapt our Perception-aware Planning framework to the real-world scenario. As shown in Fig.~\ref{fig:experiment_C}(a), a vehicle with configurations in Sec.~\ref{subsec:Implementation details} is required to navigate from point A to point B or C, during which process it has to localize itself with boxes and floor within the motion capture area.
According to Sec.~\ref{subsec:Metric Effectiveness Validation} and~\ref{subsec:Experiment in Large Scale Map}, metric $q^{(n)}$ has more holistic performance that provides reliable evaluations. Here, we show trajectories (\#2 and \#4) guided by metric $q^{(n)}$ and trajectories (\#1 and \#3) without metrics for comparison.
\vspace{-0.1cm}

\begin{figure}[h]
    \centering
    \includegraphics[width=\columnwidth]{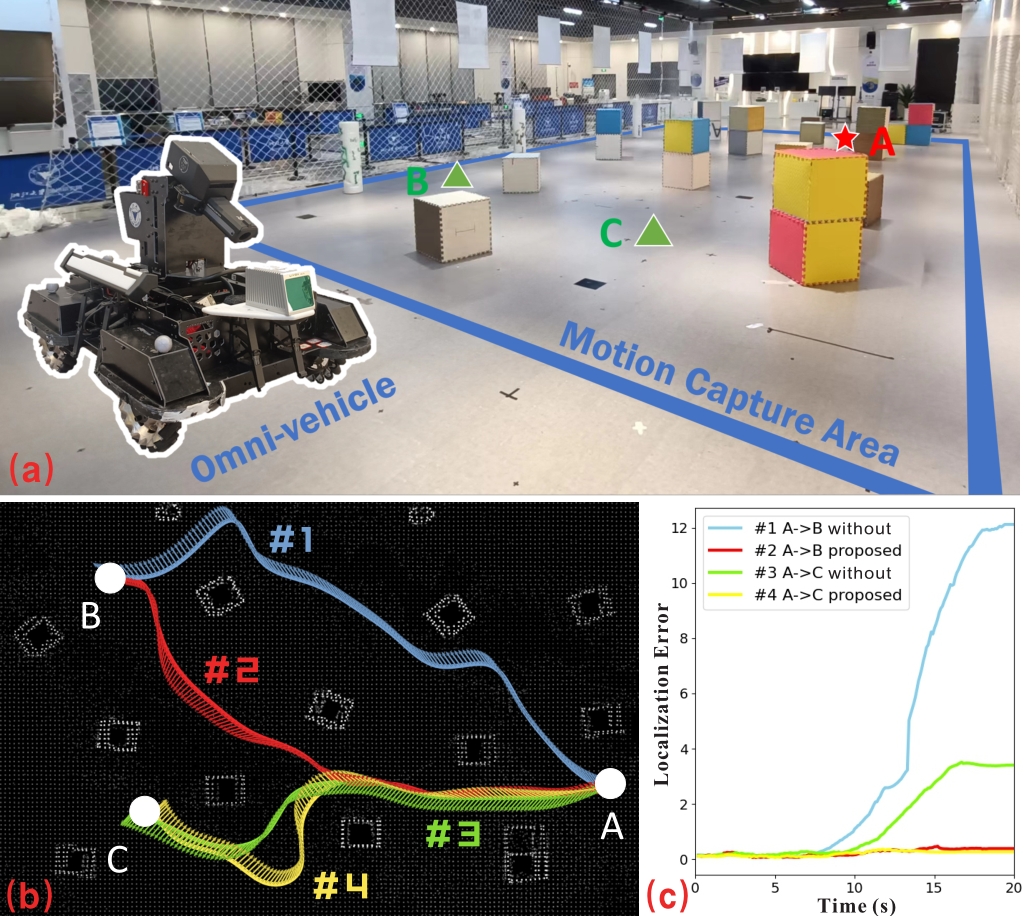}
    \caption{(a) real-world scenario, (b) generated trajectories, (c) Localization error along different trajectories.}\label{fig:experiment_C}
\end{figure}
\vspace{-0.3cm}

As shown in Fig.~\ref{fig:experiment_C}(b), the vehicle guided by $q^{(n)}$ is able to choose topologies and orientations with low observation loss.
We calculate the localization error of the robot when moving along each trajectory by $\rVert \log(\boldsymbol{T}^{-1}_{\text{gt}}\boldsymbol{T}_{\text{odom}})^\vee\lVert_2$, where $\boldsymbol{T}_{\text{gt}}$ is given by Motion Tracking System and $\boldsymbol{T}_{\text{odom}}$ is given by FAST-LIO2~\cite{fast-lio2}. As we can see in Fig.~\ref{fig:experiment_C}(c), trajectories \#2 and \#4 have lower localization error than \#1 and \#3, which demonstrates that the trajectories generated by the proposed framework are conducive to improving LAS.

\section{Conclusion}
\label{sec:Conclusion}

In this paper, we propose a LiDAR-based Perception-aware Planning framework and perturbation-induced metrics.
Through experiments, we show different features of the proposed metrics and validate the effectiveness of our framework.
At the same time, there is still space for improvement in our work.
Firstly, we assume that the localization algorithm essentially addresses a least square problem.
As a result, the proposed metrics may not be suitable for all localization algorithms, such as learning-based methods.
In addition, extensive real-world experiments across a broader range of scenarios are necessary to explore the limits of the proposed framework.
In the future, we will conduct more experiments to investigate the effectiveness of the framework and evolve the current SOLM into a dynamically updated grid structure for online replanning applications.

\appendix
\subsection{Lemma for 
Formula (\ref{eq:err_bound}) $\sim$ (\ref{eq:qo})}
\begin{lemma}
$\forall x\in[0,+\infty)$, $f_1(x), f_2(x)$ are continuously differentiable, and $f_1(x)\le f_2(x)$. If $f_1(0)=f_2(0)$, we have $f_1'(0)\le f_2'(0)$.\label{app: A}
\label{lem:1}
\end{lemma}
\begin{proof}
Let $f(x)=f_2(x)-f_1(x)$, then $\forall\zeta>0$,
\begin{equation}
    0\le\frac{f_2(\zeta)-f_1(\zeta)}{\zeta}=\frac{f(\zeta)-f(0)}{\zeta}\nonumber.
\end{equation}
Thus, 
\begin{equation}
    0\le\lim_{\zeta\rightarrow0^+}\frac{f(\zeta)-f(0)}{\zeta}=f'(0)=f_2'(0)-f_1'(0). \nonumber
\end{equation}
\end{proof}

\vspace{-0.6cm}

\subsection{Derivations of Eq.(\ref{eq:sp_errb})}
\label{app: B}
\begin{align}
\left.\frac{d \lVert \Delta A\rVert_F}{d r_{\boldsymbol k}}\right|_{r_{\boldsymbol k}=0}
&=\left.\frac{d\sqrt{\sum_{j=1}^m\lVert r_{\boldsymbol k}A^\text T\delta\boldsymbol{k}_j\rVert_2^2}}{d r_{\boldsymbol k}}\right|_{r_{\boldsymbol k}=0} \nonumber \\
&=\lim_{r_{\boldsymbol k}\rightarrow0^+}\sqrt\frac{r_{\boldsymbol k}^2\sum_{j=1}^m\delta\boldsymbol{k}_j^\text TAA^\text T\delta\boldsymbol{k}_j}{r_{\boldsymbol k}^2}\nonumber\\
&=\sqrt{\sum_{j=1}^m\delta\boldsymbol{k}_j^\text TAA^\text T\delta\boldsymbol{k}_j},
\end{align}
\begin{align}
\left.\frac{\partial \lVert \Delta \boldsymbol b\rVert_2}{\partial r_{\boldsymbol t}}\right|_{r_{\boldsymbol k}=r_{\boldsymbol t}=0}&=
\left.\frac{\partial \lVert g(\delta\boldsymbol{k})b+\delta\boldsymbol{t}\rVert_2}{\partial r_{\boldsymbol t}}\right|_{r_{\boldsymbol k}=r_{\boldsymbol t}=0}\nonumber\\
&=\lim_{r_{\boldsymbol t}\rightarrow0^+}\sqrt\frac{r_{\boldsymbol t}^2\boldsymbol{\beta}^\text T\boldsymbol{\beta}}{r_{\boldsymbol t}^2}=1,\\
\left.\frac{\partial \lVert \Delta \boldsymbol b\rVert_2}{\partial r_{\boldsymbol k}}\right|_{r_{\boldsymbol k}=r_{\boldsymbol t}=0}&=\lim_{r_{\boldsymbol k}\rightarrow0^+}\sqrt\frac{\sum_{j=1}^m\lVert r_{\boldsymbol{k}}\boldsymbol{b}^\text T\delta\boldsymbol{k}_j \rVert_2^2}{r_{\boldsymbol k}^2}\nonumber\\
&=\sqrt{\sum_{j=1}^m\delta\boldsymbol{k}_j^\text T\boldsymbol{b}\boldsymbol{b}^\text T\delta\boldsymbol{k}_j},
\end{align}

\newlength{\bibitemsep}\setlength{\bibitemsep}{0.00\baselineskip}
\newlength{\bibparskip}\setlength{\bibparskip}{0pt}
\let\oldthebibliography\thebibliography
\renewcommand\thebibliography[1]{
    \oldthebibliography{#1}
    \setlength{\parskip}{\bibitemsep}
    \setlength{\itemsep}{\bibparskip}
}
\bibliography{references}

\end{document}